\documentclass[12pt,a4paper]{article}
\usepackage{amsmath,amsthm}
\newtheorem{theorem}{Theorem}

\title{On Theorem~2.3 in ``Prediction, Learning, and Games''
by Cesa-Bianchi and Lugosi. 
}
\author{Alexey Chernov\thanks{
Supported by EPSRC grant EP/F002998/1.}\\[1ex]
\normalsize  Computer Learning Research Centre and 
Dept Computer Science\\
\normalsize  Royal Holloway, University of London,
\normalsize  Egham, Surrey TW20 0EX, UK\\
\normalsize\texttt{chernov@cs.rhul.ac.uk}
}
\date{}

\begin{document}
\maketitle

This note proves a loss bound for the exponentially weighted average
forecaster with time-varying potential,
see~\cite[\S~2.3]{CesaBianchiLugosi:2006} for context and definitions.
The present proof gives a better constant in the regret term
than Theorem~2.3 in~\cite{CesaBianchiLugosi:2006}.
This proof first appeared in~\cite{CZ2010} (Theorem~2),
where a more general algorithm is considered.
Here the proof is rewritten using 
the notation of~\cite{CesaBianchiLugosi:2006}.

\begin{theorem}\label{thm:bound}
Assume that the loss function $\ell$ is convex in the first argument
and $\ell(p,y)\in[0,1]$ for all $p\in\mathcal{D}$ and $y\in\mathcal{Y}$.
For any positive reals $\eta_1\ge\eta_2\ge\ldots$,
for any $n\ge 1$ and for any $y_1,\ldots,y_n\in\mathcal{Y}$,
the regret of the exponentially weighted average
forecaster with time-varying learning rate~$\eta_t$
satisfies
\begin{equation}\label{eq:bound}
\widehat{L}_n - \min_{i=1,\ldots,N} L_{i,n}
\le
\frac{\ln N}{\eta_n} +
   \frac{1}{8}\sum_{t=1}^n\eta_t\,.
\end{equation}
In particular, for $\eta_t=\sqrt{\frac{4\ln N}{t}}$, $t=1,\ldots,n$,
we have
$$
\widehat{L}_n - \min_{i=1,\ldots,N} L_{i,n}
\le
\sqrt{n\ln N}\,.
$$
\end{theorem}
\begin{proof}
The forecaster at step $t$
predicts $\widehat{p}_t=\sum_{i=1}^N \frac{w_{i,t-1}}{W_{t-1}}f_{i,t}$,
where $w_{i,t-1}=e^{-\eta_t L_{i,t-1}}$ and $W_{t-1} = \sum_{j=1}^N w_{j,t-1}$.
Due to convexity of $\ell$ we have
$$
\ell(\widehat{p}_t,y_t) \le \sum_{i=1}^N \frac{w_{i,t-1}}{W_{t-1}}\;\ell(f_{i,t},y_t)\,.
$$
Using the Hoeffding inequality (\cite[Lemma~A.1]{CesaBianchiLugosi:2006}), 
we get
$$
e^{-\eta_t \sum_{i=1}^N \frac{w_{i,t-1}}{W_{t-1}}\;\ell(f_{i,t},y_t)}
\ge
\sum_{i=1}^N \frac{w_{i,t-1}}{W_{t-1}}e^{-\eta_t\ell(f_{i,t},y_t) - \eta_t^2/8}
$$
and thus
\begin{equation}\label{eq:convex+Hoeffding}
e^{-\eta_t \ell(\widehat{p}_t,y_t)}
\ge
\sum_{i=1}^N \frac{w_{i,t-1}}{W_{t-1}}e^{-\eta_t\ell(f_{i,t},y_t) - \eta_t^2/8}\,.
\end{equation}

Consider the values
$$
s_{i,t-1} = e^{-\eta_{t-1}L_{i,t-1} + 
                \eta_{t-1}\widehat{L}_{t-1} - \frac{1}{8}\eta_{t-1}\sum_{k=1}^{t-1}\eta_k
              }
$$
and note that
\begin{equation}\label{eq:equalratios}
\frac{w_{i,t-1}}{W_{t-1}} = \frac{\frac{1}{N}(s_{i,t-1})^{\frac{\eta_t}{\eta_{t-1}}}}
                            {\sum_{j=1}^N \frac{1}{N}(s_{j,t-1})^{\frac{\eta_t}{\eta_{t-1}}}}\,.
\end{equation}

Let us show that $\sum_{j=1}^N\frac{1}{N}s_{j,t}\le 1$ 
by induction over $t$.
For $t=0$ this is trivial, since $s_{j,0}=1$ for all $j$.
Assume that $\sum_{j=1}^N\frac{1}{N}s_{j,t-1}\le 1$.
Then 
\begin{equation}\label{eq:denombound}
\sum_{j=1}^N \frac{1}{N}(s_{j,t-1})^{\frac{\eta_t}{\eta_{t-1}}}
\le
\left(\sum_{j=1}^N \frac{1}{N}s_{j,t-1}\right)^{\frac{\eta_t}{\eta_{t-1}}}
\le 1\,,
\end{equation}
since the function $x\mapsto x^\alpha$ is concave and monotone for $x\ge 0$ 
and $\alpha\in[0,1]$ and since $\eta_{t-1}\ge\eta_t>0$.
Using~\eqref{eq:denombound} to bound the right-hand side 
of~\eqref{eq:equalratios}, we get
$\frac{w_{i,t-1}}{W_{t-1}}\ge \frac{1}{N}(s_{i,t-1})^{\frac{\eta_t}{\eta_{t-1}}}$;
and combining with~\eqref{eq:convex+Hoeffding}, we get
$$
e^{-\eta_t \ell(\widehat{p}_t,y_t)}
\ge
\sum_{i=1}^N \frac{1}{N}(s_{i,t-1})^{\frac{\eta_t}{\eta_{t-1}}}e^{-\eta_t\ell(f_{i,t},y_t) - \eta_t^2/8}\,.
$$
It remains to note that 
$$
s_{i,t} = (s_{i,t-1})^{\frac{\eta_t}{\eta_{t-1}}}
     e^{-\eta_t\ell(f_{i,t},y_t) + \eta_t \ell(\widehat{p}_t,y_t) - \eta_t^2/8}
$$
and we get $\sum_{i=1}^N\frac{1}{N}s_{i,t}\le 1$.

For any $i$, we have
$\frac{1}{N}s_{i,n}\le\sum_{j=1}^N\frac{1}{N}s_{j,n}\le 1$,
thus
$$
-\eta_{n}L_{i,n} + \eta_{n}\widehat{L}_{n} 
- \frac{1}{8}\eta_{n}\sum_{k=1}^{n}\eta_k
\le \ln N\,,
$$
and~\eqref{eq:bound} follows.
\end{proof}

Theorem~\ref{thm:bound} recommends the learning rate 
$\eta_t=\sqrt{(4\ln N)/ t}$
instead of $\sqrt{(8\ln N)/t}$ used 
in Theorem~2.3 in~\cite{CesaBianchiLugosi:2006}
and achieves the regret term $\sqrt{n\ln N}$
instead of $\sqrt{2n\ln N} + \sqrt{0.125\ln N}$.

To compare the bounds for arbitrary learning rates,
let us observe that 
the proof of Theorem~2.3 in~\cite{CesaBianchiLugosi:2006}
actually implies
(under the assumptions of Theorem~\ref{thm:bound}):
$$
\widehat{L}_n - \min_{i=1,\ldots,N} L_{i,n}
\le
\left(\frac{2}{\eta_n} - \frac{1}{\eta_1}\right)\ln N +
   \frac{1}{8}\sum_{t=1}^n\eta_t\,.
$$
The right-hand side of this inequality
is larger than the right-hand side of~\eqref{eq:bound} 
if $\eta_n\ne\eta_1$.
If $\eta_t$ are equal for all $t$,
the bounds coincide and give 
the bound of Theorem~2.2 in~\cite{CesaBianchiLugosi:2006}.

\end{document}